\documentclass[12pt]{iopart}
\expandafter\let\csname equation*\endcsname\relax
\expandafter\let\csname endequation*\endcsname\relax
\usepackage{lipsum}
\usepackage{graphicx}
\usepackage{graphics}
\usepackage{siunitx}
\usepackage{float}
\usepackage{hyperref}
\usepackage{epstopdf}
\usepackage{enumerate}
\usepackage{cite}
\usepackage{xcolor}
\usepackage{amsmath}
\usepackage{amsthm}
\usepackage{amsfonts}
\graphicspath{{figures/}}
\usepackage{balance}
\allowdisplaybreaks
\newtheorem{theorem}{Theorem}

\usepackage{iopams}  
\begin{document}
\title[Geometrical Postural Optimisation of 7-DoF Limb-Like Manipulators]{Geometrical Postural Optimisation of 7-DoF Limb-Like Manipulators}

\author{Carlo~Tiseo,~Sydney~Rebecca~Charitos~and~Michael~Mistry}

\address{ECR, Institute of Perception Action \& Behaviour, School of Informatics, University of Edinburgh, Edinburgh, UK}
\ead{carlo.tiseo@ed.ac.uk}
\vspace{10pt}


\begin{abstract}
Robots are moving towards applications in less structured environments, but their model-based controllers are challenged by the tasks' complexity and intrinsic environmental unpredictability. Studying biological motor control can provide insights into overcoming these limitations due to the high dexterity and stability observable in humans and animals. This work presents a geometrical solution to the postural optimisation of 7-DoF limbs-like mechanisms, which are robust to singularities and computationally efficient. The theoretical formulation identified two separate decoupled optimisation strategies. The shoulder and elbow strategy align the plane of motion with the expected plane of motion and guarantee the reachability of the end-posture. The wrist strategy ensures the end-effector orientation, which is essential to retain manipulability when nearing a singular configuration. The numerical results confirmed the theoretical observations and allowed us to identify the effect of different grasp strategies on system manipulability. The geometrical method was numerically tested in thousands of configurations proving to be both robust and accurate. The tested scenarios include left and right arm postures, singular configurations, and walking scenarios. The proposed geometrical approach can find application in developing efficient and robust interaction controllers that could be applied in computational neuroscience and robotics. 
\end{abstract}

\section{Introduction} \label{sec: Introduction}
Robots intervention in our society is expanding, moving from the traditional production cells towards more unstructured application scenarios, which greatly increases the computational complexity for model-based algorithms \cite{xin2020optimization,ijspeert2013dynamical,Averta2020,ferrolho2019residual,moura2019equivalence,Wolfslag2020,yan2021decentralized,mitrovic2010adaptive,Li2018}. These methods require extended state knowledge (i.e., system and the environment states) to formulate the optimisation. Thus, any unpredicted environmental change that cannot be regarded as noise will at best compromise the action optimality or, worse, affect the system stability. Leading researchers focus on the development of optimisation methods with relaxed optimality constraints to enhance robustness\cite{xin2020optimization,ferrolho2019residual,Ferrolho2020,Kronander2016,nakanishi2011stiffness,Mastalli2020,Wolfslag2020}. 

Singular configurations are still a challenge to optimisation methods' numerical stability, but the configurations also maximise the effort routed through the two adjacent links' rigidity constraint reducing the joint torque \cite{siciliano2010robotics,moura2019equivalence}. For example, anyone can verify that walking with flexed knees requires more effort than locking these joints in the extended configuration. However, the singular configurations' numerical instability induces most legged robots to walk with flexed knees. Developing a method that is robust to singularity to identify optimal postural configurations might greatly increase the efficiency of the selected strategy.

Human limbs often inspire the mechanical structure of legged robots and manipulators with three main links and 7 degrees of freedoms (DoF), but they are often designed with a higher motion range \cite{Mastalli2020,Ferrolho2020,siciliano2010robotics}. Such design choice is mainly related to the historical need to have a general-purpose manipulating platform that could be easily reprogrammed to execute different tasks without adjusting their base placement. Notwithstanding, having a large workspace is not an issue when dealing with mobile robots where a limited range of motion is compensated by its ability to move its own base. 

\begin{figure}[t]
    \centering
    \includegraphics[width=\textwidth,trim=0cm 0.5cm 0cm 0.5cm, clip]{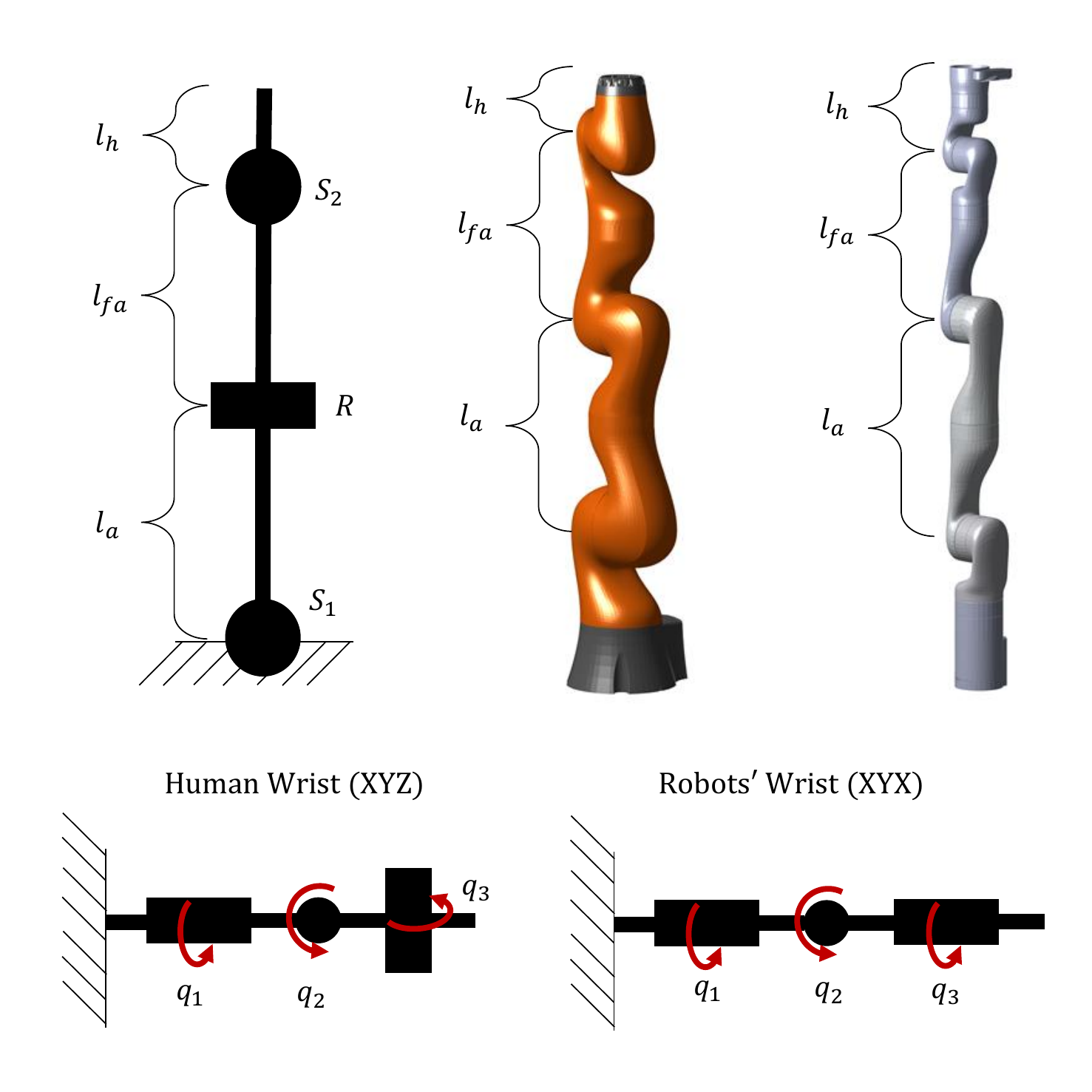}
    \caption{A generic 7 degrees of freedoms (DoF) mechanism with aligned Spherical-Revolute-Spherical joint configurations is on the left, and the equivalent representations of a Kuka IIWA14 and a Kinova Gen 3 are on the centre and the rightmost, respectively. The mechanisms all have a spherical ($S_1$) joint connecting the robot base to the arm ($l_\text{a}$). A revolute joint (R) connecting the arm and the forearm ($l_\text{fa}$). A second spherical joint ($S_2$) connecting the forearm to the hand ($l_\text{a}$). The main difference between these manipulators and the human arm is their $\mathrm{XYX}$ wrist's joints configuration; meanwhile, the human wrist's joints configuration is $\mathrm{XYZ}$.}
    \label{fig:Kinematics7dof}
\end{figure} 

This paper aims to identify a geometrical solution to postural optimisation for 7-DoF mechanisms that could control these manipulators without employing numerical optimisation.

\section{Preliminaries}
The proposed method exploits the kineto-static duality to identify an optimal pose for 7-DoF mechanisms. Currently, solving this problem requires solving either inverse kinematics or non-linear dynamics optimisation \cite{siciliano2010robotics,moura2019equivalence,Angelini2019}. However, these methods have some intrinsic limitations due to the projection matrices involved in the optimisation problem formulation. This section provides an overview of the knowledge required to understand and contextualise the proposed method\cite{aghili2015projection,nakanishi2008international,de2005operational,featherstone1997load, khatib1983dynamic,dietrich2018hierarchical,de2006task,dehio2018modeling}.

The Fractal Impedance Controller is a recently introduced control framework, enabling interaction control in redundant manipulators without relying on null-space projections \cite{tiseo2020,babarahmati2019,babarahmati2020,Tiseo2020Bio,tiseo2020Planner}. The FIC is an asymptotically stable conservative field that can be defined in non-linear spaces (e.g., quaternion, spherical), enabling to define simply connected energetic manifolds in generalised coordinates \cite{tiseo2020,babarahmati2019,babarahmati2020,Tiseo2020Bio,tiseo2020Planner}. Once the task-space controller has been calibrated for the system's physical properties, it is globally stable, and it does not require observing the null-space energy to guarantee its stability\cite{babarahmati2019,Tiseo2020Bio}. 

\subsection{Model-Based Optimisation of Redundant Robots}
Model-based optimisation finds applications in multiple branches of engineering, including planning and control for robotics applications. In the eighties and early nineties, the development of null-space projections enabled the exploitation of robots' redundancy by defying cost functions capable of selecting optimal solutions for a set of optimisation criteria \cite{moura2019equivalence,de2005operational,featherstone1997load, khatib1983dynamic,dietrich2018hierarchical}. The steep increase in computational power that we have experienced in the last three decades enabled the solution of complex non-linear optimisation problems, enabling robots' development in unstructured environments \cite{xin2020optimization,Xin2020,ferrolho2019residual,Ferrolho2020,Mastalli2020}. Notwithstanding, these methods still have robustness issues related to numerical instability close to singularities and models' inaccuracies \cite{moura2019equivalence,xin2020optimization}. There are multiple formulations derived over the years and the details about their derivations are not needed for the scope of this manuscript. Thus, we remand the reader to the specific literature on the subject, which can be found in \cite{moura2019equivalence,de2005operational,featherstone1997load, khatib1983dynamic,dietrich2018hierarchical,xin2020optimization}.

\subsubsection{Null-Space Projections}
Null-Space projectors are mathematical transformations that define a subset of the system domain in generalised coordinates (i.e., joint coordinates) that, ideally, do not project information (i.e., energy) in the task-space. They are deployed to formulate the robots' task-space models used for kinematics and dynamics optimisation problems that aim to minimise joint space (i.e., joint torques) without affecting the task-space action. Over the years, multiple formulations for the projectors have been proposed. These methods have been proven to be equivalent solutions, but they might have different numerical stability and complexity due to different numerical conditioning. The general formulation of the inverse projection for an $n$-dimensional system in generalised coordinates (i.e., $q\in \mathcal{R}^n$) is:
\begin{equation}
\label{eq:ProjInvDyn}
P=(I-A^\#A)
\end{equation}
where $I \in \mathcal{R}^{n\times n}$ is an identity matrix, $A$ is the extended Jacobian for the system and the task constraints, and $A^\#$ is a generalised inverse of $A$. 
The extended Jacobian identifies a desirable posture by introducing arbitrary constraints in the tangent space of the robot's manifold, which can be projected in generalised coordinates using null-space projections \cite{xin2020optimization}. The constraints can be introduced by different tasks, organised in a hierarchical structure based on the task priority. However, hierarchical projection limits the solution of domain with a higher risk of incurring numerical instability \cite{xin2020optimization,Wolfslag2020,moura2019equivalence}. 

In summary, the null-space projections can be used to formulate optimisation algorithms to maximise the alignment between the expected task and the geodesics of the robot's manifold, which minimise the energy expenditure in the task. 

A common trend in formulations that exploit null-space projections begin with the separation of the dynamics in generalised coordinates into two separate equations before developing different formulations \cite{moura2019equivalence,Xin2020}. The first equation describing the null-space projection of the constraints:
\begin{equation}
\label{eq:JSDynamics}
P M\ddot{q}+ Ph = \tau_{c} 
\end{equation}
where $M$ is the joint space inertia matrix, $h$ is the non-linear dynamics and $\tau_\text{c}$ are the joint torque producing a task-space motion. The second equation describing the component of the dynamics generating a motion affecting the task:
\begin{equation}
\label{eq:NSDynamics}
(I-P) M\ddot{q}+ (I-P)h + \tau_\text{NS}= \tau_{m}
\end{equation}
where $I$ is an identity matrix, $\tau_\text{NS}$ are the torques generated by the tasks' constraints in the Null-Space, and $\tau_{m}$ are the joint torque required to control the null-space. Therefore, the total torque command can be written as $\tau=\tau_{c}+\tau_{m}$ and implies that the controllers' observer has to track the energy in both sub-spaces to guarantee stability. 

\subsubsection{Implications of Projected Dynamics on Port-Hamiltonian and Whole-Body controllers Stability}
Port-Hamiltonian controllers are commonly known as Admittance and Impedance controllers\cite{hogan2018impedance,Averta2020,Tiseo2020Bio,Tiseo2018}. Their task-space implementations often rely on null-space projections to fully control and optimise the robot posture \cite{Averta2020,babarahmati2019,xin2020optimization}. However, null-space projections can be substituted in back-derivable systems by introducing a weak potential field in the joint space pulling the robot towards a reference configuration \cite{Averta2020}. 

Whole-body controllers are used in planning and control loco-manipulation tasks in mobile robotics, where null-space projections are used to formulate the system equations used to generate the control commands\cite{xin2020optimization,Xin2020,Mastalli2020,tiseo2018modelling,tiseo2018bipedal}. 

Port-Hamiltonian and whole-body controllers stability can be compromised by singular configurations or the deterioration of the extended Jacobian ($A$) rank in both methods, implying that  $P$ is not a valid base for the two sub-spaces defined in \autoref{eq:JSDynamics} and \autoref{eq:NSDynamics} \cite{moura2019equivalence,Xin2020,dietrich2018hierarchical}. Consequently, it compromises the possibility of performing the numerical integration required by the controllers' observers to integrate the power ($\dot{E}=\nabla E(x)\dot{x}$) into energy ($E(x)$). 

Singularity cases are less critical because they can be addressed by reducing the workspace. Thus, singularities are mainly responsible for removing optimal solutions involving the exploitation of singular configurations. Meanwhile, the deterioration of the extended Jacobian ($A$) rank can be caused either by the failure of any assumption made for the tasks' modelling or the system equations becomes linearly dependent, generating numerical instability \cite{moura2019equivalence,Kronander2016}. These events are much more difficult to detect in highly variable environments and play a key role in reducing these methods' robustness. Nonetheless,  multiple methods have been proposed to extend their robustness by increasing the margin of stability by introducing slack to the solution and/or refining the optimisation problem \cite{xin2020optimization,Xin2020,Wolfslag2020,Angelini2019,yan2021decentralized,Li2018,ferrolho2019residual,Ferrolho2020}. These solutions often introduce higher computational cost, and they retain a degree of susceptibility to the failure of the assumptions made for the model formulation.

\subsection{Kineto-Static \& Postural Optimisation} \label{sec: Method}
The kineto-static duality is a well-known method that exploits the orthogonality between the ability of kinematic chains to generate a force or a velocity, which can be evaluated from the mechanism's geometrical Jacobian \cite{siciliano2010robotics}. The geometrical Jacobian ($J$) describes the relationship between the end-effector tangential velocity and joint velocities \cite{siciliano2010robotics}. Thus, the rigidity constraints of the kinematic chain are mainly acting along the orthogonal direction, which is also the most efficient direction to apply force in a given configuration. The kineto-static duality is expressed by the following two equations which applies to the virtual work principle between joint-space and task-space of the robot.
\begin{equation}
\begin{array}{rl}
  \dot{x}=   & J\dot{q}  \\
  \tau=      & J^T h_\text{e}
\end{array}
\label{eq:KinetoStaticDuality}    
\end{equation}
where $\dot{x}$ is the Cartesian space velocity, $\dot{q}$ is the joint space velocity, $\tau$ are the joint torques, and $h_\text{e}$ is the Cartesian Force.

\subsection{Manipulability and Force Ellipsoids \& Polytopes}
The manipulability and force ellipsoids are geometric representations of the system capabilities obtained using the base of the tangent space obtained via differential kinematics and the kineto-static duality in \autoref{eq:KinetoStaticDuality} \cite{siciliano2010robotics}. The manipulability ellipsoid is obtained by studying the solution of the characteristic polynomial obtained from the differential kinematics when the joint velocities belong to a unit hyper-sphere.
\begin{equation}
\dot{q}^T\dot{q}=\dot{x}^T(JJ^{T})^{-1}\dot{x}=1
\label{eq:ManElli}    
\end{equation}
\autoref{eq:ManElli} describes the deformation of the hyper-sphere when projected through the robot kinematics. Thus, the eigenvalues ($\zeta_\text{m} \in \mathcal{R}^n$) of the matrix $(JJ^{T})^{-1}$ define the scaling factor of the hyper-sphere in the task-space direction described by the associated eigenvector. 

The force ellipsoids are similarly obtained starting from a torque unit hyper-sphere in the joint space leading to the following equation.
\begin{equation}
\tau^T\tau=h_\text{e}^T(JJ^{T})h_\text{e}=1
\label{eq:ForceElli}    
\end{equation}
\autoref{eq:ForceElli} describes the deformation of the torque hyper-sphere when projected through the robot kinematics. Consequently, the force ellipsoids have the axis directions (i.e., eigenvectors) of the manipulability ellipsoids. In contrast, the eigenvalues ($\zeta_\text{f} \in \mathcal{R}^n$) are equal to the inverse of the respective eigenvalues of manipulability ellipsoids.
\begin{equation}
\zeta_{\text{f}i}=\zeta_{\text{m}i}^{-1}~\forall~i\in\left[1,n\right], ~i \in\mathcal{N}
\label{eq:eigenvalues}    
\end{equation}

For a long time, this property has been exploited in numerical optimisation to identify the better posture that aligns the task requirements with the eigenvectors defined by \autoref{eq:ManElli} and \autoref{eq:ForceElli} to maximise the robot performances. 

Recently, polytopes have been applied in the formulation of the optimisation problem in robotics to increase robustness \cite{Wolfslag2020,ferrolho2019residual,Ferrolho2020}. They are a more accurate representation of the robot capabilities than the ellipsoid, but they are also computationally more expensive \cite{ferrolho2019residual}. Differently from the ellipsoids, they are usually computed using the torque limits of the robot. Notwithstanding, the ellipsoids are an inner approximation of the polytope obtained from a unit hyper-sphere in the joint space \cite{ferrolho2019residual}. Therefore, this polytope provides a better description of the robot capabilities (for all the Cartesian directions) that are not aligned with the eigenvectors determined by the robot Jacobian. However, it converges to the same values of the ellipsoid when intersecting the eigenvectors \cite{ferrolho2019residual}. 

\subsection{The Fractal Impedance Controller}
The Fractal Impedance Controller is a recently introduced passive controller that generates an asymptotically stable force field around the desired state \cite{babarahmati2019, Tiseo2020Bio}. The FIC is robust delays and low-bandwidth in the feedback loop  due to its path independent energy, implying that multiple controllers can be added in parallel or series without affecting stability \cite{babarahmati2019,babarahmati2020,tiseo2020Planner}. Its autonomous harmonic trajectories have upper-bounded energy and maximum power, guaranteeing global asymptotic stability in all the work-space in fixed-base robots and asymptotic stability within the base of support in mobile robots \cite{Tiseo2020Bio,tiseo2020Planner}. The controller has been successfully deployed in manipulation, teleoperation, human-robot cooperation, model predictive control for motion planning, and computational neuroscience \cite{babarahmati2019,babarahmati2020,tiseo2020, Tiseo2020Bio,tiseo2020Planner,tiseo2021,tiseo2021exploiting}. 

The stability characteristics of the FIC allows the formulation of stable controllers for redundant manipulators without using null-space projections \cite{tiseo2020}. However, to remove them, an additional task-space controller acting on the elbow generates virtual mechanical constraints enabling the complete postural control of the robot. However, the method still employed kinematic postural optimisation to identify the desired posture of the robot. The experiments with a Kuka LWR4 proved that such a method could accurately execute an interaction control on a 7-DoF robotic arm without knowing the robot's dynamics. Although this interaction controller does not incur a catastrophic failure due to the singularity and numerical instability of the kinematic optimisation, they still impede the exploitation of a particular configuration \cite{tiseo2020}. 

\subsection{The Task Separation Principle}
The task Separation Principles (SP) has been proposed as an explanation of how human motor control handles redundancy of the body \cite{guigon2007computational,tommasino2017extended,tommasino2017task,tiseo2018modelling}. Despite it being implemented using different formulations, they all agree there are two separate controllers in the human nervous system. The first controller is a static controller that handles quasi-static predictable force fields (e.g., gravity and body intrinsic mechanical impedance) and drives the postural optimisation exploiting the body redundancy. Meanwhile, the second controller is handling a velocity-dependent force field and perturbations \cite{guigon2007computational}. 

A well-known method for its implementation is the combination of non-linear inverse optimisation and impedance controllers to generate stable force fields controlling both the task and the null-space. Tommasino and Campolo recently develop a Passive Motion Paradigm (PMP) implementation of the SP, known as $\lambda_0-$PMP \cite{tommasino2017extended,tommasino2017task}. They proved that a postural strategy could be regressed from human movement data via non-linear inverse optimisation to generate a holonomic manifold (in the generalised coordinates), rendering the controller's energy path independent. These methods are similar to the dynamics primitives concept developed from robotics, where a robot learns to control its state via the regression of an optimal command strategy from the desired task-space behaviour \cite{flash2005motor,ijspeert2013dynamical,tommasino2017task,tommasino2017task}. Consequently, these approaches to the SP produce task dependant control strategies, and similarly to the dynamics primitives, they do not provide a general solution. Another option to formulate the SP is the optimal control implementations that offer a better generality at a higher cost due to online optimisation. Guigon proposed an optimal control architecture for the SP in \cite{guigon2007computational}, capable of generating control commands containing both the static and dynamic controller output, which is less expensive compared to earlier versions. 

The authors have recently proposed an SP architecture that exploits the stability properties of the FIC to generate human-like reaching movements on a planar arm \cite{tiseo2021}. We have also shown that a similar architecture can be used in quaternion coordinates to control wrist pointing coordinates \cite{tiseo2021exploiting}. The desired Cartesian trajectories were controlled by projecting the pointing direction in spherical coordinates. Meanwhile, the redundancy DoF introduced by the wrist pronosupination is controlled by the rotation around the pointing direction. Unlike previous methods, its exploitation of the FIC stability allows control over redundant manipulators without employing optimisation or regressing a movement strategy from human data. 

\section{A Geometrical Solution to Postural Optimisation for 7-DoF Robots}
The algebraic and geometrical solution to inverse kinematics problems are well known to robotics and deployed in low-dimensional non-redundant manipulators, despite similar inverse kinematics for 7-DoF manipulator with aligned Spherical-Revolute-Spherical joint structures \cite{siciliano2010robotics,benati1982inverse,gong2019analytical}. Some examples are: human limbs, manipulators using a DLR manipulator (Kuka, Baxter, Kinova, etc.), and humanoids robots limbs (e.g. NASA Valkyrie). 

These methods have limited application due to the need to employ numerical optimisation to plan and control task-space motions in redundant robots. The recent development in task-space control and planning made using the FIC, discussed in the previous section, open the possibility of controlling this type of manipulators without requiring numerical optimisation. To do so, a method to identify the desired posture is required to generate the task-space controllers that are capable of fully controlling the robot.

This section will show that it is sufficient and necessary to optimise the robot posture to maximise a specific alignment of the tangent space base (eigenvectors of the matrix $JJ^T$). We will introduce a hierarchical numerical optimisation for this problem and a specific geometrical solution for the postural optimisation of 7-DoF manipulators.  

\subsection{Postural Optimisation via the Geometric Jacobian}
The geometric Jacobian describes the properties of the robot tangent space. As shown by the kineto-static duality, it describes both the capability of the system of generating effort (i.e., force and torques) and flow (i.e., velocities) at the end-effector. The analysis of the eigenvector of the matrix $JJ^T$ are a base frame for the robot manifold tangent space, while the eigenvalues associated with manipulability and force ellipsoids are metrics for these properties. Therefore, accounting for the direction of end-effector flow ($\hat{v}$) and maximum effort ($\hat{h}_\text{e}$) is a sufficient condition for the identification of an optimal posture for a specific task. This property can be easily verified from the robot dynamics in a generalised coordinate space, which shows that the geometric Jacobian is present in the description of all the elements with the exception of the null-space torque. 
\begin{theorem}
Let our system be a redundant, fully actuated backdrivable mechanism controlled by a stable superimposition of task-space controllers; then an optimal posture is a necessary and sufficient condition for motion optimality. 
\end{theorem}
\begin{proof}
The robot dynamics is described starting from the robot's Lagrangian.
\begin{equation}
    \mathcal{L}=\mathrm{U}-\mathrm{K}
\end{equation}
where $\mathrm{U}$ is the potential energy and $\mathrm{K}$ is the kinetic energy \cite{siciliano2010robotics}. The force field associated with the energy manifold is described by the gradient of the Lagrangian $\mathcal{L}$. It is worth noting that is equal to $\nabla\mathcal{L}=\tau=\tau_c+\tau_m$ obtained by the summation of \autoref{eq:JSDynamics} and \autoref{eq:NSDynamics}.
\begin{equation}
\nabla \mathcal{L}=M\dot{q}+ B+ \sum_{i=1}^{N_\text{CoM}} J^T_{\text{CoM}i} F_{\text{g}i} + \sum_{j=1}^{N_\text{World}} J^T_{\text{e}j} h_{\text{e}j} +  \sum_{j=1}^{N_\text{Ctr}} J^T_{\text{C}k} F_{\text{C}k} + \tau_\text{NS}=0
\label{dynamicsEq}   
\end{equation}
where M is the inertia matrix, B is the non-linear dynamics, $\sum J^T_{\text{CoM}} F_{\text{g}}$ describes the gravity, $\sum  J^T_{\text{e}} F_{\text{e}}$ are the interactions with the world (i.e., contacts), $\sum J^T_{\text{C}} F_{\text{C}}$ is the contribution of all the controllers active in the robot, and $ \tau_\text{NS}$ are the null space torques and mechanical losses. It is worth nothing that both M and B also contain the geometric Jacobian in their formulation. Equation \ref{dynamicsEq} can be rewritten once again as:
\begin{equation}
\nabla \mathcal{L}=\nabla U_{MB}+ \sum_{i=1}^{N_\text{CoM}} J^T_{\text{CoM}i} \nabla U_{\text{g}i} + \sum_{j=1}^{N_\text{World}} J^T_{\text{e}j} \nabla U_{\text{e}j} +  \sum_{j=1}^{N_\text{Ctr}} J^T_{\text{C}k} \nabla U_{\text{C}k} + \nabla U_\text{NS}=0
\label{dynamicsEqV2}   
\end{equation}
where we expressed the mechanical efforts in the equation as a gradient of their energetic manifolds.

Considering that $\nabla U_{MB}$ is the gradient of the robot kinetic energy \cite{siciliano2010robotics}, we can now write the instantaneous power of the robot dynamics as:
\begin{equation}
\nabla \mathcal{L}\dot{q} = \left(\sum J^T_i \nabla U_{TSi}\dot{x}_{i}\right) + \nabla U_\text{NS}\dot{q} =\left(\sum J^T_i \nabla U_{TSi} J_i\dot{q}\right) + \nabla U_\text{NS}\dot{q}
\label{Power}   
\end{equation}
This equation clearly shows that the optimal solution (i.e., minimum energy expenditure) is obtained by having a perfect projection of the task-space power in the joint space (i.e., $\nabla U_\text{NS}\dot{q}=0$), which is associated with a static equilibrium in the null-space. In other words, the optimal strategy nullifies any joint-space movement that is not required to be performed by the task. Considering that \autoref{Power} describes the system power, and the robot's posture determines the only controllable components via the geometric Jacobian. Therefore, an optimal posture is a sufficient and necessary condition for optimality.
\end{proof}

Despite postural optimisation being the sole guarantor of optimality, it also imposes directional constraints. Meanwhile, the feasibility and the stability of a task also require tracking the robot's energy and limiting the magnitude of the required power. These conditions are frequently added as constraints in the optimisation problem, and are often related to increased optimisation costs \cite{Averta2020,xin2020optimization,dietrich2018hierarchical,Ferrolho2020,Mastalli2020,yan2021decentralized,Wolfslag2020}. However, the FIC provides trajectory independent energy tracking and limits the power; thus, postural optimisation is necessary and sufficient for having a stable optimal control on an architecture based on the FIC \cite{babarahmati2019,tiseo2020}.

\subsubsection{Postural Optimisation via Numerical Optimisation}
The hierarchical postural optimisation for a generic manipulator can be formulated as follows:
\begin{equation}
\begin{array}{l}
    \min \left(\hat{f}_\text{e}^T J_\text{P}\left(q\right) J_\text{P}^T\left(q\right) \hat{f}_\text{e}\right)  \\
    \begin{array}{ll}
\text{subject to:}&\\
& ||\mathcal{K}\left(q\right)-x_\text{d}||\le \lambda\\
& \arg \max \left(\hat{v}^T J_\text{P}\left(q\right) J^T_\text{P}\left(q\right) \hat{v}\right) 
    \end{array}
\end{array}
\label{NumOptimisation}
\end{equation}
where $\hat{f}_\text{e}$ is the direction of the end-effector task's force, $J_\text{P}$ are the first 3 rows of the Jacobian describing the linear components, $\mathcal{K}$ is the direct kinematics, $x_\text{d}$ is the desired end-effector pose, $\lambda$ determines the task accuracy, and $\hat{v}$ is the direction of the end-effector task's velocity. Alternatively, the hierarchical optimisation could be formulated prioritising the velocity generation by inverting the order, resulting in:
\begin{equation}
\begin{array}{l}
     \max \left(\hat{v}^T J_\text{P}\left(q\right) J^T_\text{P}\left(q\right) \hat{v}\right) \\
    \begin{array}{ll}
\text{subject to:}&\\
& ||\mathcal{K}\left(q\right)-x_\text{d}||\le \lambda\\
& \arg  \min \left(\hat{f}_\text{e}^T J_\text{P}\left(q\right) J^T_\text{P}\left(q\right) \hat{f}_\text{e}\right) 
    \end{array}
\end{array}
\label{NumOptimisationB}
\end{equation}
The other two options for the optimisation are using the direction of the end-effector twist ($\hat{t}$) or wrench ($\hat{w}$), where there is no need for performing a hierarchical optimisation because they are 6D-dimensional vectors. The optimisation for the twist is:
\begin{equation}
\begin{array}{l}
    \max \left(\hat{t}^T J\left(q\right) J^T\left(q\right) \hat{t}\right)  \\
    \begin{array}{ll}
\text{subject to:}&\\
& ||\mathcal{K}\left(q\right)-x_\text{d}||\le \lambda 
    \end{array}
\end{array}
\label{TwistOptimisation}
\end{equation}
The optimisation for the wrench is:
\begin{equation}
\begin{array}{l}
    \min \left(\hat{w}^T J\left(q\right) J^T\left(q\right) \hat{w}\right)  \\
    \begin{array}{ll}
\text{subject to:}&\\
& ||\mathcal{K}\left(q\right)-x_\text{d}||\le \lambda
\end{array}
\end{array}
\label{WrenchOptimisation}
\end{equation}
\begin{figure}[!ht]
    \centering
    \includegraphics[width=\textwidth,trim=0cm 6cm 2cm 2cm, clip]{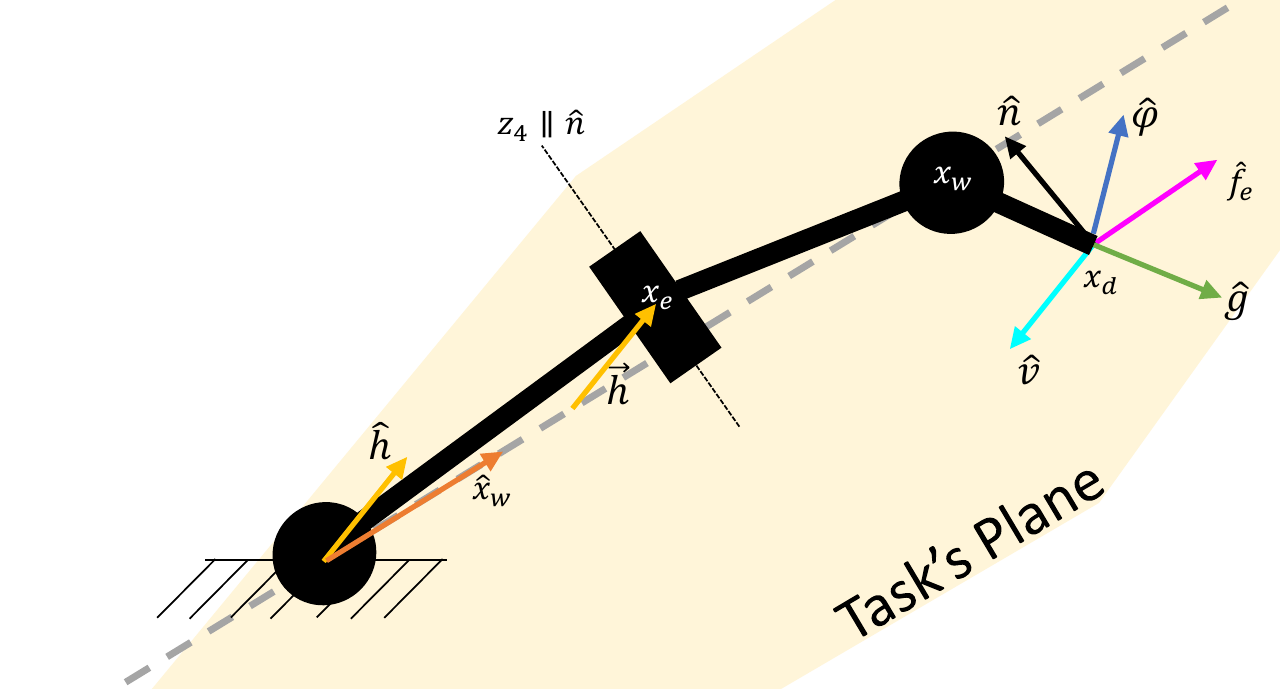}
    \caption{The task's force ($\hat{f}_\text{e}$) and velocity ($\hat{v}$) directions define the task's plane defining the normal direction ($\hat{n}$). The shoulder joint maximises the alignment of the elbow joints ($z_4$) with $\hat{n}$. The $x_\text{w}$ is determined moving back from $x_\text{d}$ in the grasp direction ($\hat{g}$) for a length equal to length of the third link ($l_\text{h}$). The pronosupination $\hat{\phi}$ is a redundant degree of freedom controlling the orientation of the end-effector around $\hat{g}$.}
    \label{fig:Concept}
\end{figure}

\subsubsection{Geometrical Optimisation for 7-DoF Kinematics}
Mechanisms that have a Spherical-Revolute-Spherical joint arrangement are 3-D versions of a 3-link planar arm with the ability to change the plane of motion. Therefore, the inverse kinematics and postural optimisation are the same of a 3-link arm once the direction for the end-effector force and velocity are assigned, as exemplified in \autoref{fig:Concept}. Consequently, the posture can be optimised only and only if there are no requirements for the grasping direction. Whereas, if the grasping direction is defined, the remaining two configurations are of the left and the right arm. This fact can be easily experimentally verified by taking a full cup of a liquid, choosing a grasping direction in space and trying arbitrary to change the elbow position while moving the cup without spilling the liquid. It goes without saying that the use of any of the additional 20-DoF provided by the hand is forbidden in the experiment. Lastly, it shall be noted that the left and right arm redundancy is also present for any other solution of the proposed optimisation method. However, it is not a problem because switching from one to the other will require special manoeuvring to extend the elbow joint fully. Furthermore, the range of motion is limited to one of the two configurations in human beings.

Given a mechanism having three links of lengths $l_\text{a}$, $l_\text{fa}$ and $l_\text{h}$, respectively. Being $ \hat{n}=\hat{v}\times\hat{f}_\text{e}$ the unit vector orthogonal to the task's plane. The optimal posture can be identified deriving the positions of the wrist $x_\text{w}$ and elbow $x_\text{e}$. 
\begin{equation}
    x_\text{w}=x_\text{d}-l_\text{h}\hat{g}
    \label{WristPosition}
\end{equation}
where $x_\text{d}$ is the desired end-effector position, and $\hat{g}$ is the grasping direction. The derivation of elbow position is immediate if $||x_\text{w}||\ge l_\text{a}+l_\text{fa}$ being $x_\text{e}=l_\text{a}\hat{x}_\text{d}$, where $\hat{x}_\text{d}$ is the direction of the vector connecting the mechanism base frame to the desired end-effector. The position of the elbow of a left arm can be derived as follows:
\begin{equation}
\begin{array}{l}
        \vec{h}=\hat{x}_\text{w}\times n\\
        \hat{h}=\cfrac{\vec{h}}{||\vec{h}||}\\
        k=\cfrac{l_\text{fa}}{l_\text{a}}\\
        l_\text{a}^{'}=\cfrac{||x_\text{w}||}{2}+\cfrac{l_\text{a}^2(1-k^2)}{2||x_\text{w}||}\\
        \mu=\sqrt{l_\text{a}^{'~2}-l^2}\\
      x_\text{e}=\begin{cases}
            -\mu\hat{h}+l_\text{a}^{'}\hat{x}_\text{w}\text{, }\left(\hat{h}_2<0~\&~x_\text{d}\ge0\right) \wedge \left(\hat{h}_2\le0~\&~x_\text{d}\le0\right) \\
            +\mu\hat{h}+l_\text{a}^{'}\hat{x}_\text{w}\text{, Otherwise}
            \end{cases}
\end{array}
\label{geomOptim}
\end{equation}
where $\hat{h}_i$ is the $i^{th}$ element of the unit vector $\hat{h}$. The right arm formulation of \autoref{geomOptim} can be obtained inverting the signs of $\mu\hat{h}$ between the two conditions of $x_\text{e}$. As mentioned earlier, the solution still presents the redundancy for $\hat{h}_2=0$, which implies a top and a bottom solution, which has to be treated on a case by case basis. However, if we consider an arm moving in the Earth's gravitational field, the lower elbow posture is usually preferable.
  
In the event of an unconstrained grasping direction, the problem can be rewritten considering that the optimal configuration will be the alignment of both the hand and the forearm with the force direction ($\hat{g}=\hat{x}_\text{w}$). Consequently, the \autoref{geomOptim} can be rewritten as follows:
\begin{equation}
\begin{array}{l}
        \vec{h}=\hat{x}_\text{d}\times n\\
        \hat{h}=\cfrac{\vec{h}}{||\vec{h}||}\\
        k=\cfrac{l_\text{fa}+l_\text{h}}{l_\text{a}}\\
        l_\text{a}^{'}=\cfrac{||x_\text{d}||}{2}+\cfrac{l_\text{a}^2(1-k^2)}{2||x_\text{d}||}\\
        \mu=\sqrt{l_\text{a}^{'~2}-l^2}\\
            x_\text{e}=\begin{cases}
            -\mu\hat{h}+l_\text{a}^{'}\hat{x}_\text{w}\text{, }\left(\hat{h}_2<0~\&~x_\text{d}\ge0\right) \wedge \left(\hat{h}_2\le0~\&~x_\text{d}\le0\right) \\
            +\mu\hat{h}+l_\text{a}^{'}\hat{x}_\text{w}\text{, Otherwise}
            \end{cases}\\
        x_\text{w}=l_\text{fa}\cfrac{x_\text{d}-x_\text{e}}{||x_\text{d}-x_\text{e}||}+x_\text{e}
\end{array}
\label{geomOptimB}
\end{equation}
It is worth noting that in the case that $||x_\text{d}||\ge l_\text{a}+l_\text{fa}+l_\text{h}$ the elbow position is $x_\text{e}=l_\text{a}\hat{x}_\text{d}$, and the wrist position is $x_\text{w}=l_\text{fa}\hat{x}_\text{d}+x_\text{e}$. 

\subsection{The Inverse Kinematics of Human Limbs}
\begin{figure}[!ht]
    \centering
    \includegraphics[width=\textwidth,trim=0cm 0cm 0cm 0cm, clip]{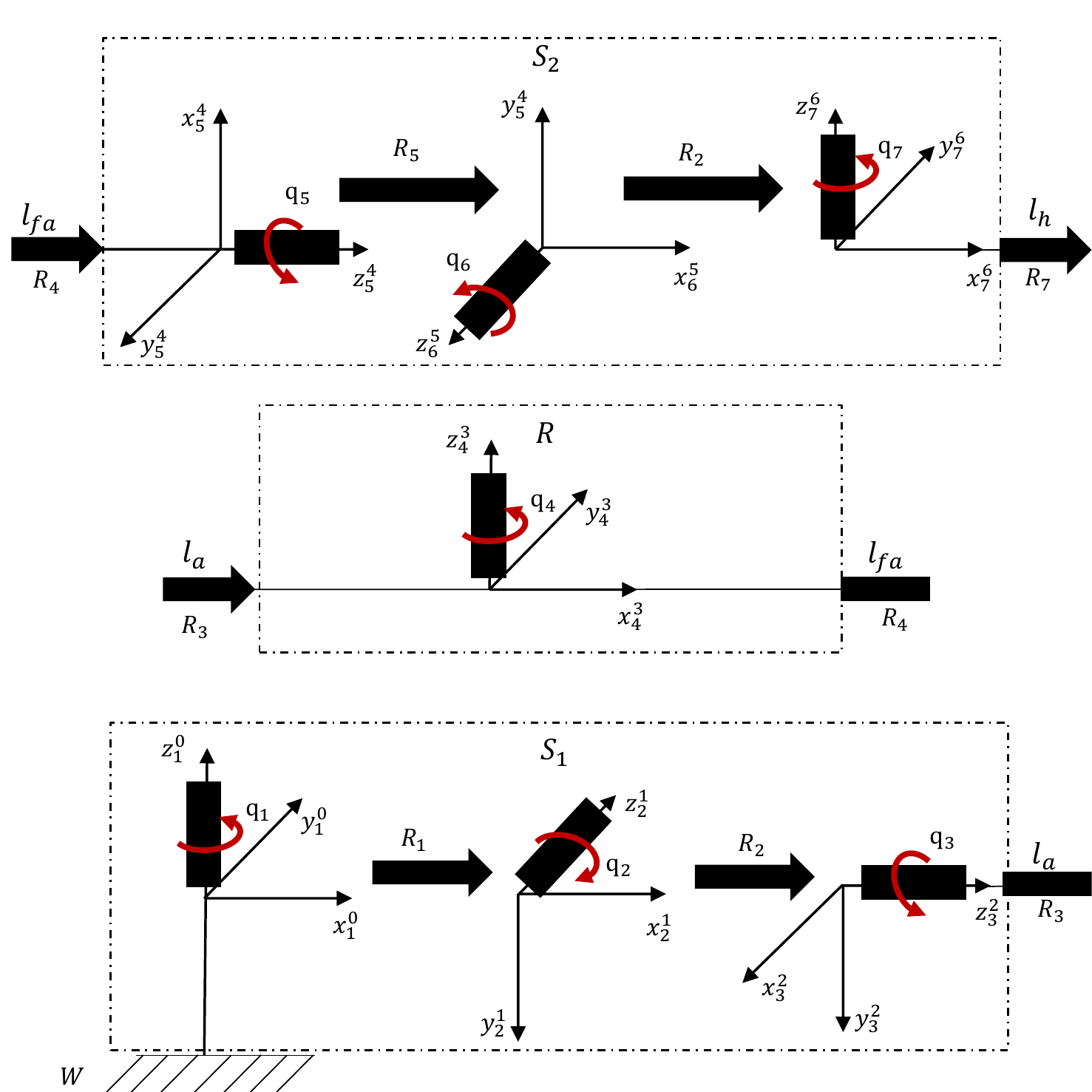}
    \caption{The $S_1$, $R$, and $S_2$ from \autoref{fig:Kinematics7dof} have been exploded to show the reference frames used for the formulation of the direct kinematics. $S_1$ is described with a $\mathrm{ZYX}$ Euler angles transformation and is connected to the world reference (W). Meanwhile, $S_2$ is described by a $\mathrm{XYZ}$ Euler angles transformation.}
    \label{fig:KinematicsChain}
\end{figure}
Human limbs differ from the majority of manipulators due to the joints' arrangement of the wrist, which have a $\mathrm{XYX}$ configuration, that allows to decouple the end-effector roll from the wrist pronosupination. In contrast, human joints are organised as $\mathrm{XYZ}$ that couples the end-effector roll to the pronosupination, as shown in \autoref{fig:KinematicsChain}. The different joint arrangement also implies that the Inverse Kinematic (IK) solution identified for these robot cannot be directly applied to identify the joints' configuration. Let's define $\hat{\phi}$ as the unit vector describing the desired wrist torsion, which is required to define the limb posture. The equations for the shoulder joints are:
\begin{equation*}
    \begin{array}{l}
         q_1=\arctan(\hat{x}_\text{e}\cdot R^0_{1y},\hat{x}_\text{e}\cdot R^0_{1x})  \\
         R_1=R^0_1\left(\begin{array}{ccc}
              \cos\left(q_1\right)&-\sin\left(q_1\right) &0 \\
              \sin\left(q_1\right)&\cos\left(q_1\right)&0\\
              0&0&1\end{array} \right)\\
         R^1_2=\left(R_{1x}~~-R_{1z}~~R_{1y}\right)
    \end{array}
\end{equation*}
\begin{equation}
    \begin{array}{l}
         q_2=\arctan(\hat{x}_\text{e}\cdot R^1_{2y},\hat{x}_\text{e}\cdot R^1_{2x})\\
         R_2=R^1_2\left(\begin{array}{ccc}
              \cos\left(q_2\right)&-\sin\left(q_2\right) &0 \\
              \sin\left(q_2\right)&\cos\left(q_2\right)&0\\
              0&0&1\end{array} \right)\\
        R^2_3=\left(-R_{2z}~~R_{2y}~~R_{2x}\right)
    \end{array}
    \label{eq:IKShoulder}
\end{equation}
\begin{equation*}
    \begin{array}{l}
         q_3=\arctan(\hat{x}_\text{e}\cdot R^2_{3y},\hat{x}_\text{e}\cdot R^2_{3x})\\
        R_3=R^2_3\left(\begin{array}{ccc}
              \cos\left(q_3\right)&-\sin\left(q_3\right) &0 \\
              \sin\left(q_3\right)&\cos\left(q_3\right)&0\\
              0&0&1\end{array} \right)\\
       R^3_4=\left(R_{3z}~~-R_{3x}~~-R_{3y}\right)
    \end{array}
\end{equation*}
where $q_i$ is the $i^{th}$ joint angle, $R_i$ is the rotation matrix of the base from of the $i^{th}$ link, and $R_i^{i-1}$ is the base frame of the $i^{th}$ joint.

The equation for the elbow joint is:
\begin{equation}
    \begin{array}{l}
       q_4=\arctan((\hat{x}_\text{w}-\hat{x}_\text{e})\cdot R^3_{4y},(\hat{x}_\text{w}-\hat{x}_\text{e})\cdot R^3_{4x})\\
        R_4=R^3_4\left(\begin{array}{ccc}
              \cos\left(q_4\right)&-\sin\left(q_4\right) &0 \\
              \sin\left(q_4\right)&\cos\left(q_4\right)&0\\
              0&0&1\end{array} \right)\\
       R^4_5=\left(R_{4z}~~-R_{4y}~~R_{4x}\right)\\
    \end{array}
    \label{eq:IKElbow}
\end{equation}
The equation for the wrist joints are:
\begin{equation*}
    \begin{array}{l}
       \alpha=\hat{g}\times(\hat{\phi}\times{\hat{g}})\\
       q_5=\arctan(\alpha\cdot R^4_{5y},\alpha\cdot R^4_{5x})\\
        R_5=R^4_5\left(\begin{array}{ccc}
              \cos\left(q_5\right)&-\sin\left(q_5\right) &0 \\
              \sin\left(q_5\right)&\cos\left(q_5\right)&0\\
              0&0&1\end{array} \right)
    \end{array}
    \label{eq:Wrist}
\end{equation*}

\begin{equation}
    \begin{array}{l}
       q_6=\arctan(\hat{g}\cdot R^5_{6y},\hat{g}\cdot R^5_{6x})\\
        R_6=R^5_6\left(\begin{array}{ccc}
              \cos\left(q_6\right)&-\sin\left(q_6\right) &0 \\
              \sin\left(q_6\right)&\cos\left(q_6\right)&0\\
              0&0&1\end{array} \right)\\
       R^6_7=\left(R_{6x}~~-R_{6z}~~R_{6y}\right)
    \end{array}
\end{equation}

\begin{equation*}
    \begin{array}{l}
       q_7=\arctan(\hat{g}\cdot R^6_{7y},\hat{g}\cdot R^6_{7x})\\
       R_7=R^6_7\left(\begin{array}{ccc}
              \cos\left(q_7\right)&-\sin\left(q_7\right) &0 \\
              \sin\left(q_7\right)&\cos\left(q_7\right)&0\\
              0&0&1\end{array} \right)\\
    \end{array}
\end{equation*}
\section{Method Characterisation}
The links' lengths use in all experiments are $l_{a}=\SI{.37}{\meter}$, $l_{fa}=\SI{.32}{\meter}$, and $l_{h}=\SI{.10}{\meter}$. The experiments were performed in Matlab 2020a, using an Intel i7-7700HQ with \SI{16}{\giga B} of RAM.The proposed postural optimisation method has been tested in multiple scenarios to evaluate robustness and computational performances. 

\subsection{Robustness and Computational Performances}
In the first instance the ability to exploit the redundancy with and without assigning $\hat{g}$ in both the left and right arm configurations is analysed. To do so we assign a desired end-effector position $x_\text{d}=\left(0.5~~0~~0\right)\si{\meter}$, and randomly generate $500$ velocity unit vectors ($\hat{v}$) using the random function. The force direction ($\hat{f}_\text{e}$) is taken orthogonal to $\hat{v}$. The grasp direction $\hat{g}=\left(1~~0~~0\right)$ when it is provided as input to the algorithm. The data shown in \autoref{fig:LRArms} indicate the proposed method is capable of exploiting the redundancy for both the left and right arm configurations. The second robustness test evaluates the method stability in computing postured for $1000$ randomly assigned end-effector positions. The point is chosen from a spherical domain having a radius $1.2$ times bigger than the maximum reach of the arm. The $\hat{v}$ and $f_\text{e}$ are selected similarly to the previous experiment. The grasp $\hat{g}$ direction is assigned randomly. The total computational time for all the postures is \SI{30}{\milli \second}, which implies an average time for a single posture of \SI{30}{\micro\second}. \autoref{fig:ArmSphere} shows all the solutions generated during this latest test.
\begin{figure}[t]
    \centering
    \includegraphics[width=\textwidth,trim=0cm 0cm 0cm 0cm,clip]{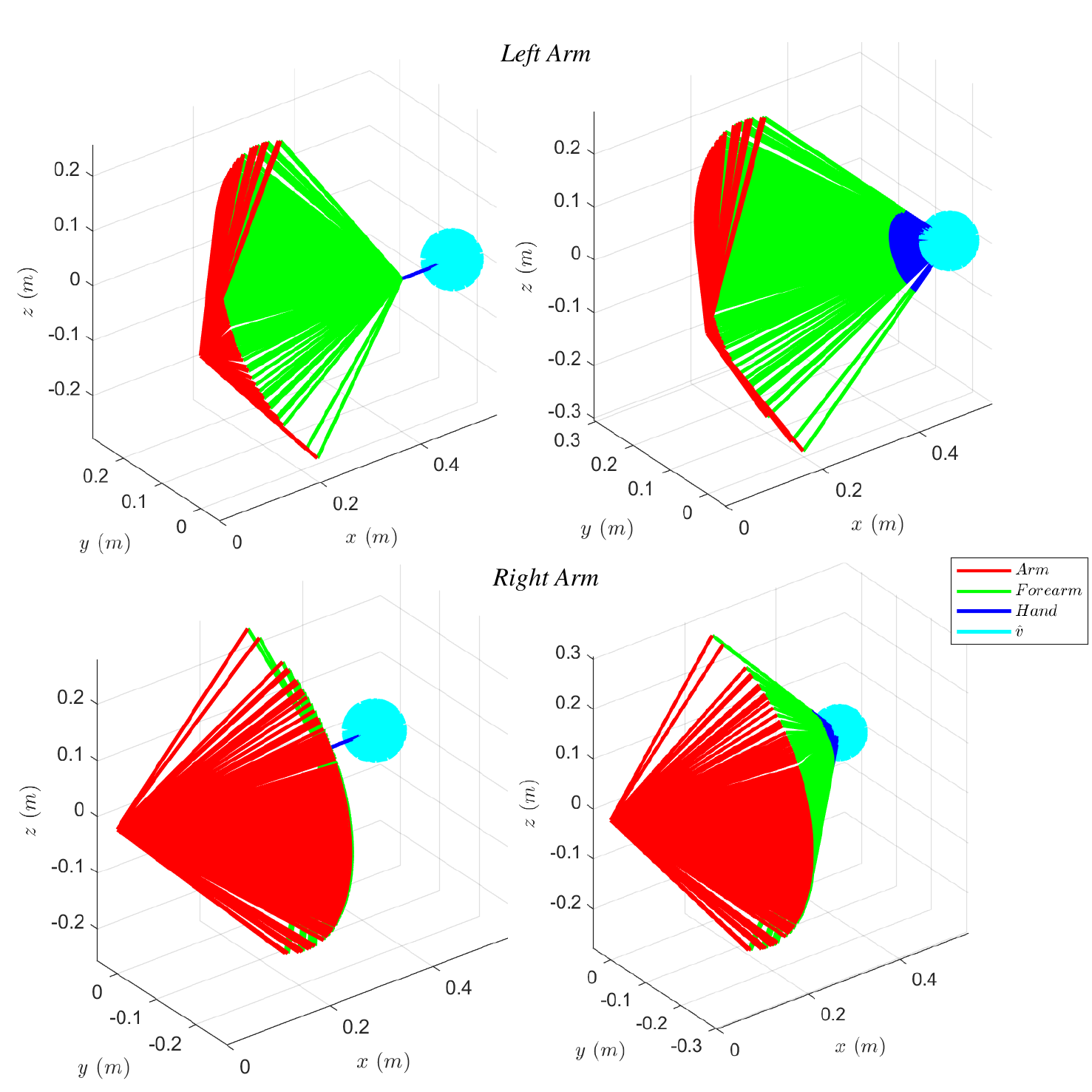}
    \caption{Left and Right arms redundant postures obtained for $x_\text{d}=\left(0.5~~0~~0\right)\si{\meter}$ associated with a random unit vector $\hat{v}$. On the left column the grasp direction is fixed to $\hat{g}=\left(1~~0~~0\right)$ and generated using \autoref{geomOptim}. On the right column, the grasp direction is not assigned ($\hat{g}=\left(0~~0~~0\right)$) and the postures are generated using \autoref{geomOptimB}.}
    \label{fig:LRArms}
\end{figure} 

\begin{figure}[t]
    \centering
    \includegraphics[width=.9\textwidth,trim=3cm 9cm 4.5cm 9cm,clip]{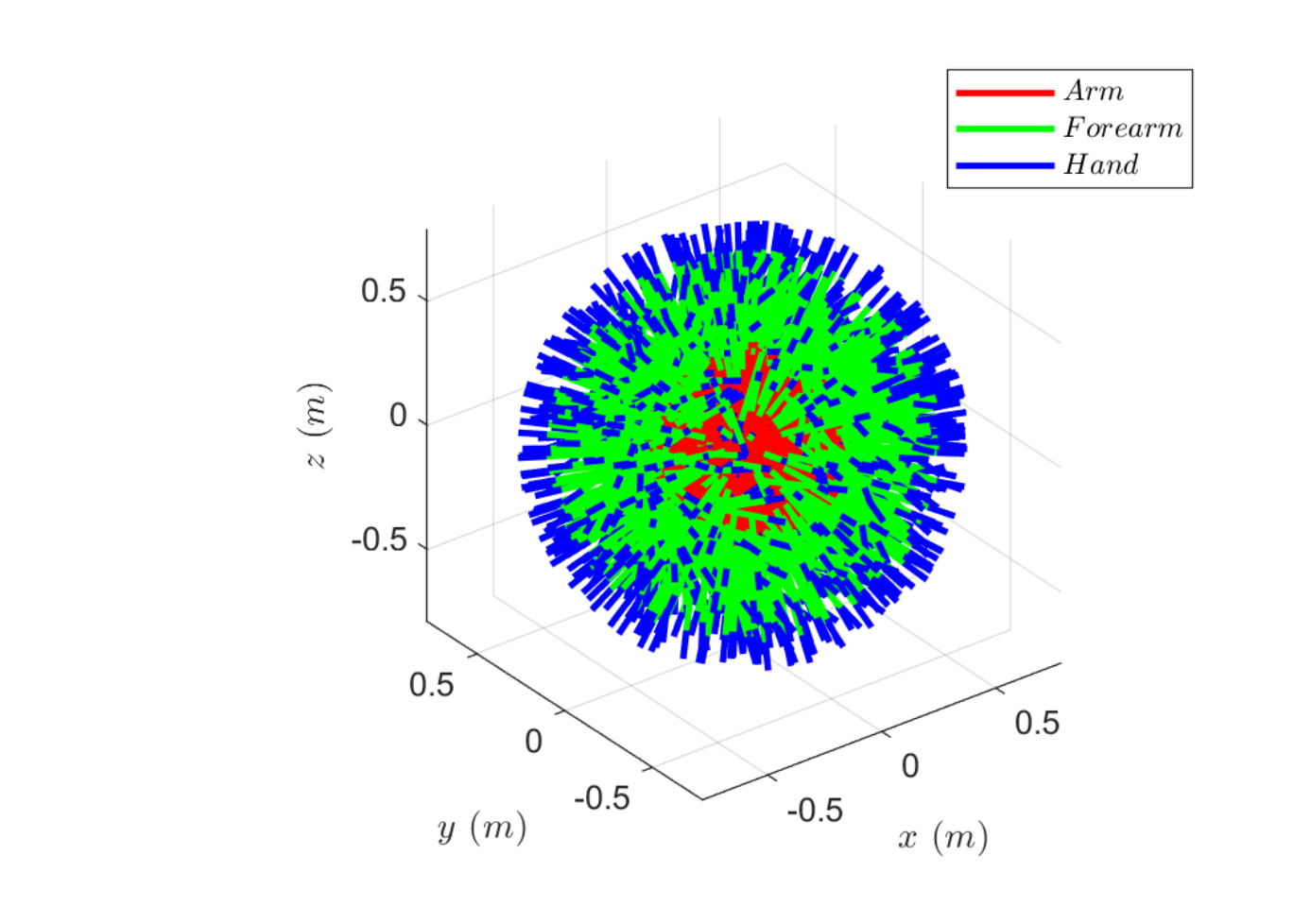}
    \caption{Visualisation of the $1000$ random end-effector positions used to evaluate the robustness of the proposed method. The results show that no error occurred and they were computed in \SI{30}{\milli \second}. }
    \label{fig:ArmSphere}
\end{figure} 

\subsection{Geometrical Inverse Kinematics}
To evaluate the accuracy of the proposed IK, $1000$ end-effector positions ($x_\text{d}$) are generated and associated to a random set of orthogonal $\hat{v}$ and $\hat{f}_\text{e}$ of unit vector. The grasp direction ($\hat{g}$) are chosen aligned to $\hat{v}$, while the desired hand pronosupination are chosen to align $\hat{\phi}$ with $\hat{f}_\text{e}$. The postural optimisation is used to derive the desired elbow ($x_\text{e}$) and wrist ($x_\text{w}$) positions. The error position error at the elbow, wrist and hand are evaluated in combination with the orientation errors for the grasp and the pronosupination.
\begin{equation}
    \label{IKerror}
    \begin{array}{l}
         E_\text{E}=\left\|x_\text{e}-x_\text{IK-e}\right\|   \\
         E_\text{W}=\left\|x_\text{w}-x_\text{IK-w}\right\|   \\
         E_\text{H}=\left\|x_\text{d}-x_\text{IK-d}\right\|    \\
         E_{\hat{g}}=\arccos(R_{7x}\cdot\hat{g})   \\
         E_{\hat{\phi}}=\arccos(R_{7z}\cdot\hat{\phi})  
    \end{array}
\end{equation}
where $x_\text{IK-e}=R_3 \left(0~~0~~l_\text{a}\right)^T$, $x_\text{IK-w}=R_4 \left(l_\text{fa}~~0~~0\right)^T+x_\text{IK-e}$ and $x_\text{IK-d}=R_7 \left(l_\text{h}~~0~~0\right)^T+x_\text{IK-w}$. The recorded values for all the mean and standard deviations of all the errors are negligible, showing that the proposed method is highly accurate. The cumulative computational time for the $1000$ postural optimisation and IK problems are \SI{78}{\milli\second}, which implies an average computational time of \SI{78}{\micro\second} per end-effector position.

\subsection{Grasp Selection \& Postural Optimisation}
Having established that the manipulability of these types of robots are equivalent to the planar 3-link arm once the adequate plane of motion is selected. It implies that manoeuvrability is mainly affected by the distance of the end-effector from the robot base. This section evaluates of the impact of the grasp strategy ($\hat{g}$) on the robot manipulability for the different extension of the arm ($d$). We consider the following three cases:
\begin{enumerate}
    \item N-Grasp: absence of a grasp strategy (\autoref{geomOptimB})
    \item F-Grasp: $\hat{g}$ aligned with the force direction ($\hat{g}=-\hat{f}_\text{e}$)
    \item V-Grasp: $\hat{g}$ aligned with the velocity direction ($\hat{g}=\hat{v}$)
\end{enumerate}
The manipulability properties for each desired end-effector position considering 36 equispaced orthogonal combinations of velocity $\hat{v}$ and force $\hat{f}_\text{e}$, which implies $\angle\hat{g}\hat{x}=\theta_\text{g} \in \left[0,360\right]\si{deg}$. The metrics used for the comparison is the effect on the values is:
\begin{equation}
    C= \cfrac{\hat{f}_\text{e}^T J_\text{P}J_\text{P}^T\hat{f}_\text{e}}{\hat{v}^T J_\text{P}J_\text{P}^T\hat{v}}
    \label{eq:C}
\end{equation} 
The results in \autoref{fig:C-grasp} indicate that the scaling of these values are mainly driven from the distance from the robot base. The curves also show that the V-Grasp reaches the best results in reducing $(\hat{f}_\text{e}^T J_\text{P}J_\text{P}^T\hat{f}_\text{e})$ while increasing $(\hat{v}^T J_\text{P}J_\text{P}^T\hat{v})$, especially in the extended positions where controlling the grasp direction is essential to retain some manipulability. G-grasp is recommendable when there is no predominant directionality of the task, probably due to the neutral configuration of the wrist joints. Lastly, F-Grasp has the worst performance due to a more remarked directionality than the other two cases without providing any significant benefit. A clear example when the V-Grasp is a preferable strategy is locomotion, where the limb is mostly pushing on the ground and working close to the singularity, as shown in \autoref{fig:WalkingExample}. The results further verify the robustness to the singularity of the proposed method. More importantly, these results clarify the difference between ($\hat{f}_\text{e}$, $\hat{v}$) and ($\dot{x}_d$, $\hat{h}_\text{e}$). ($\hat{f}_\text{e}$, $\hat{v}$) are the desired orientation of the physical property of the robot; meanwhile ($\dot{x}_d$, $\hat{h}_\text{e}$) describe physical variables at the end-effector. For example, in \autoref{fig:WalkingExample}, $\hat{v}$ is the orthogonal direction with respect to the expected ground reaction force, and it is derived aligning $\hat{f}_\text{e}$ with the expected direction of the ground interaction forces. Meanwhile, $\dot{x}_d$ describes the direction of motion of the leg in space. It is worth noting that $\hat{v}$ and $\hat{f}_\text{e}$ do not have to be orthogonal to identify a suitable strategy, and if they are parallel, an orthogonal vector can be used to define the plane of motion.

\begin{figure}[t]
    \centering
    \includegraphics[width=\textwidth,trim=2cm 6cm 4.5cm 6cm,clip]{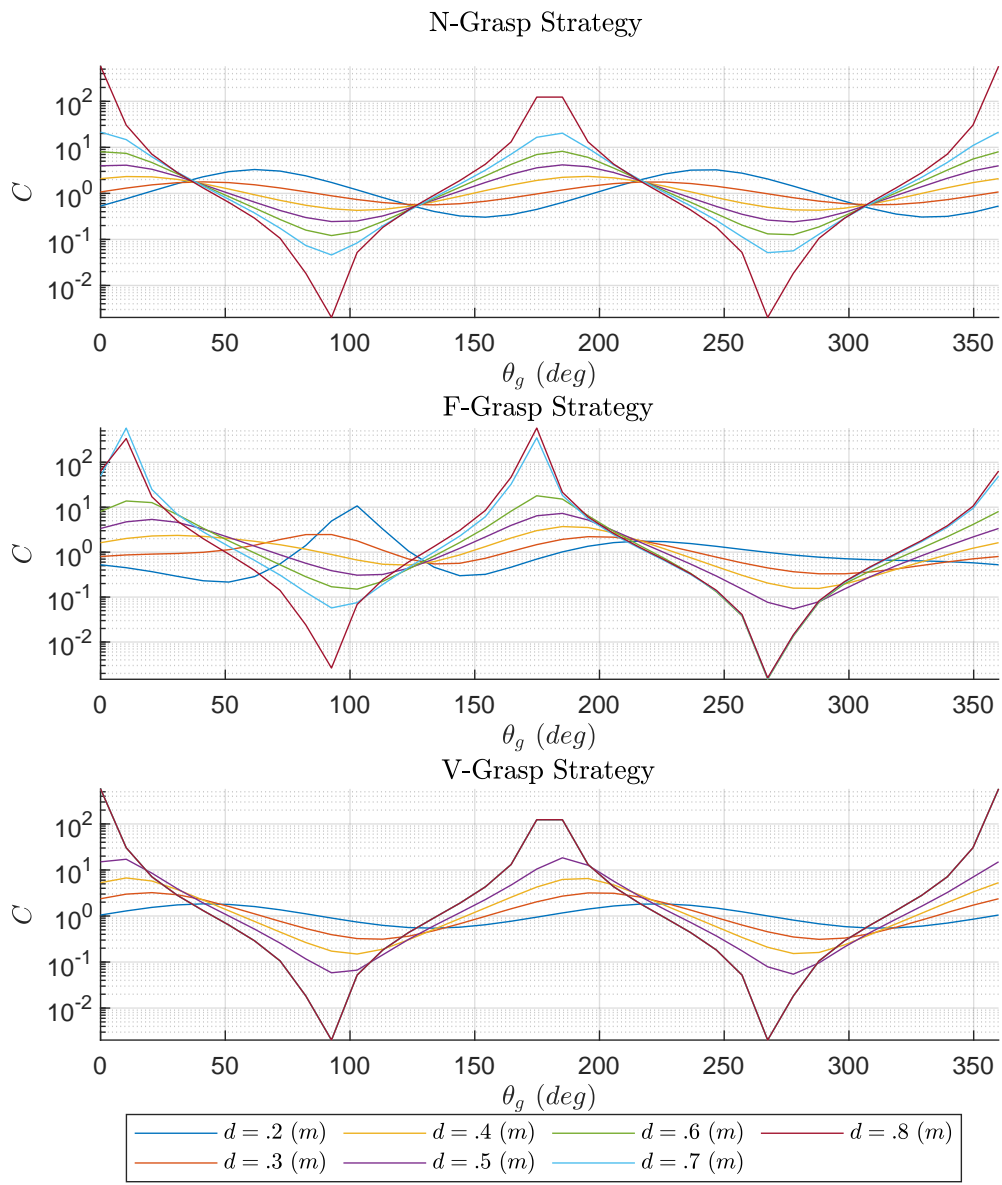}
    \caption{C is defined in \autoref{eq:C} and describes the ratio between the force and manipulability characteristics based on the direction of interaction of $\hat{v}$ and $\hat{f}_\text{e}$. Therefore, a value of 1 indicates a balanced performance; the peaks indicate unbalanced performances. A value closer to zero is preferable because it implies a minimisation of the numerator and a maximisation of the denominator of C. The N-Grasp and V-Grasp have smoother transitions compared to F-Grasp. They are also less susceptible to the velocity direction. Meanwhile, the F-Grasp transitions are steeper, and it has a more pronounced asymmetry with higher minima.}
    \label{fig:C-grasp}
\end{figure} 

\begin{figure}[t]
    \centering
    \includegraphics[width=\textwidth,trim=2cm 8.5cm 3cm 8.5cm,clip]{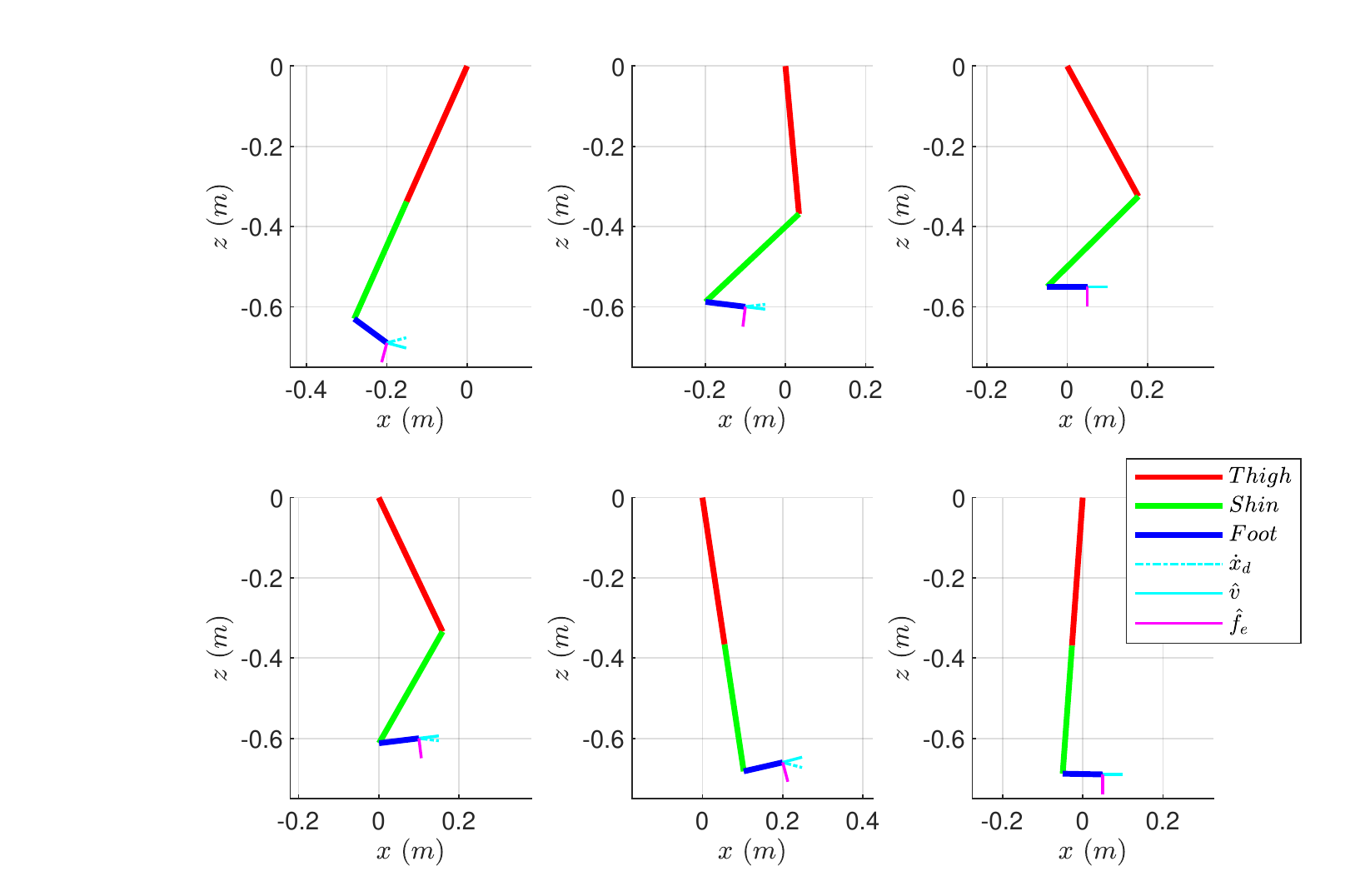}
    \caption{Six snapshots of a walking trajectory are shown: the toe-off strategy is on the top-left corner. It is followed by three intermediate stages of the swing trajectory showing the ankle inversion before showing the heel-strike and the middle point of the support phase. It shows how a proper selection of $f_\text{e}$ and $\hat{v}$ allows exploiting the singularity while retaining manipulability. An indicative end-effector velocity $\dot{x}_\text{d}$ is also shown to help visualise its difference from $\hat{v}$.  }
    \label{fig:WalkingExample}
\end{figure}

\section{Discussion}
The results confirm that the proposed method is a robust solution for postural optimisation and inverse kinematics on limb-like mechanisms. The data in \autoref{fig:LRArms} prove that the proposed method can be applied to either the right or left arm configuration. \autoref{fig:ArmSphere}, \autoref{fig:C-grasp} and \autoref{fig:WalkingExample} confirmed the robustness of the geometrical postural optimisation to singularity. The effects of different grasp strategies on the system manipulability in \autoref{fig:C-grasp} shows that in the absence of task constraints to $\hat{g}$, it is recommendable to keep the wrist in a neutral position (N-Grasp) on the centre of the limb reach. Meanwhile, V-Grasp should be employed in contracted and extended postures to improve the manipulability at the end-effector, as shown in  \autoref{fig:WalkingExample}. The computational time recorded in our simulations indicate computational times in the order of hundreds of microseconds in Matlab, which could be potentially reduced even further by implementing the code using a language such as C++. 

The presented geometrical approach provides an alternative to the inversion of the robot dynamics and numerical optimisation for solving the redundancy problem in 7-DoF manipulators with an aligned Spherical-Revolute-Spherical joints' configuration. Although this method can simplify the problem for current planning and control algorithms relying on inverse projections, the major benefit will come from exploiting the system singular configuration that requires removing the inverse matrices from the problem formulations. Our earlier work shows that it is possible to control redundant manipulators in the task-space without relying on inverse projections \cite{tiseo2020}, and future integration with the proposed postural optimisation might prove beneficial for interaction robustness in robotics. Furthermore, \autoref{Power} proves that an optimal postural optimisation is a sufficient and necessary condition for optimality. The main limitation is the absence of joint constraints. It could be easily added in the IK formulation by introducing a saturation of the joint variables, but it cannot be directly accounted for in the postural optimisation. However, this can be addressed by using a mapping of the mechanism workspace to identify accessible posture. Furthermore, this is not an issue for the task-space impedance controllers deployed in \cite{tiseo2020,tiseo2021}, which will reach the minimum distance from the desired pose by entering in mechanical equilibrium with the kinematic constraints of the mechanism.

Another major implication of our results on the postural optimisation of limb-like mechanisms is deterministic, and it is fully defined based on the selection of the motion plane and the grasp constraints ($\hat{g}$ and $\hat{\phi}$). Therefore, different tasks can be classified by the principal direction of their energetic manifold, describing the flow of energy of the environment (i.e. power). As a consequence, the optimal posture is the one that adequately aligns the tangent space of the robot manifold described by \autoref{Power} with the task, that in the 7-DoF manipulators considered in this manuscript implies aligning the arm plane with the $\hat{v}$ and $\hat{f_\text{e}}$ defined by the task. This theoretical result finds confirmation in experimental results from motion capture available in the literature \cite{tiseo2018bipedal,tiseo2018modelling}. They proved that the human walking strategy could be derived by analysing the principal direction of the gravitational attractor acting on the centre of mass. They also used this observation to define a close form equation for propagating the mechanical wave in the gravitational field that was used to plan locomotion trajectories in bipeds and quadrupeds \cite{tiseo2018strange,Tiseo2018,Tiseo2019}. These implications could explain the SP and its impact on learning and generalising motor skills. \autoref{fig:C-grasp} shows the grasp strategy has a high impact on manipulability in areas close to the fully extended and contracted configurations. However, its influence is reduced in the middle of the arms reach. Thus, it indicates that there are two sets of tasks involved in controlling and optimising the posture of the limbs. The task assigned to the shoulder/hip and the elbow/knee mainly determines the alignment of the limb's kineto-dynamic characteristics with the task. The wrist/ankle strategy exploits the redundancy to adjust and compensate for local characteristics and perturbations. As mentioned earlier, PMP architectures that model the SP and generate human-like behaviour in a planar 3-link arm during reaching motion and in a spherical wrist for pointing tasks without null-space projectors \cite{tiseo2021,tiseo2021} have been proposed. These methods could be integrated with the proposed geometrical methods to control 7-DoF manipulators without requiring numerical optimisation.

\section{Conclusion}
We have proposed and verified geometrical solutions for the postural optimisation and inverse kinematics of limb-like mechanisms. The method is computationally efficient and robust to singularities, and it only relies on the knowledge of the task's plane. These characteristics make it relevant to robotics to improve the robustness and efficiency of interaction and computational neuroscience to develop methods that explain how the nervous system can control movements, learn new tasks, and generalise its knowledge.  Our future work will focus on integrating these control architectures with the presented postural optimisation to control a 3-D arm in performing human-like movements to verify if the nervous system implements a similar approach to motor control.  

\section{Acknowledgements}
This work has been supported by EPSRC UK RAI Hub ORCA (EP/R026173/1), National Centre for Nuclear Robotics (NCNR EPR02572X/1) and THING project in the EU Horizon 2020 (ICT-2017-1).
\vspace{1cm}
\section*{References}
\bibliography{main}
\bibliographystyle{IEEEtran} 

\end{document}